\documentclass[a4paper,12pt]{article}

\usepackage[utf8]{inputenc}
\usepackage{algorithm}
\usepackage{algorithmicx}
\usepackage{algpseudocode}
\usepackage{url}
\usepackage{amsmath}
\usepackage{amsthm}
\usepackage{subfig}
\usepackage{cite}

\newtheorem{theorem}{Theorem}

\usepackage{pgfplots}
\pgfplotsset{compat=1.14}
\usepgfplotslibrary{external} 
\newcommand{\newalg}[1]{{\color{red}#1}}
\newcommand{\onell}{(1+(\lambda,\lambda))}
\newcommand{\lambdabound}{\overline{\lambda}}

\begin{document}

\title{The 1/5-th Rule with Rollbacks: On~Self-Adjustment of the Population Size
in~the~$(1+(\lambda,\lambda))$~GA\footnote{An extended two-page abstract of this work will appear in proceedings of the Genetic and Evolutionary Computation Conference, GECCO'19.}}

\author{Anton Bassin, Maxim Buzdalov}

\maketitle

\begin{abstract}
Self-adjustment of parameters can significantly improve the performance of evolutionary algorithms.
A notable example is the $(1+(\lambda,\lambda))$ genetic algorithm, where the adaptation of the population size helps to achieve the linear runtime on the \textsc{OneMax} problem.
However, on problems which interfere with the assumptions behind the self-adjustment procedure, its usage can lead to performance degradation compared to static parameter choices.
In particular, the one fifth rule, which guides the adaptation in the example above, is able to raise the population size too fast on problems which are too far away from the perfect
fitness-distance correlation.

We propose a modification of the one fifth rule in order to have less negative impact on the performance in scenarios when the original rule reduces the performance.
Our modification, while still having a good performance on \textsc{OneMax}, both theoretically and in practice, also shows better results on
linear functions with random weights and on random satisfiable MAX-SAT instances.
\end{abstract}

\section{Introduction}

The $(1+(\lambda,\lambda))$ genetic algorithm, proposed in~\cite{learning-from-black-box-thcs}, is a bright example of a successful
application of self-adjustment of parameters. Not only it is the first example of an evolutionary algorithm with linear runtime on
a simple benchmarking function \textsc{OneMax}, as all previously known algorithms were $\Omega(n \log n)$, it is also the first
example that self-adjustment of a parameter can be asymptotically better than any fixed parameter choice, possibly depending on the
problem's characteristics.

However, despite first successes on optimizing problems other than \textsc{OneMax}~--- notably, good enough running times on linear functions
with random weights taken from $[1;2]$ or royal road functions reported in~\cite{learning-from-black-box-thcs}, or a suprisingly good
performance in practice when optimizing MAX-SAT problems~\cite{goldman-punch-ppp}, which was subsequently supported theoretically~\cite{buzdalovD-gecco17-3cnf}~---
this algorithm is quite slow to conquer other territories. A possible explanation of this fact is that the method of parameter adjustment
can behave badly on problems not very similar to \textsc{OneMax}, that is, those with low fitness-distance correlation.
This is partially supported by the analysis in~\cite{buzdalovD-gecco17-3cnf}, where it was necessary to bound the parameter from the above
to ensure good performance.

This paper aims at investigating this problem in more detail. After describing the algorithm and the analysed problems in Section~\ref{sect:pre},
we perform a little landscape analysis for a few problems in Section~\ref{sect:eui} and discuss the possible reasons for why the
$(1+(\lambda,\lambda))$~GA is not so good and how it relates to its self-adjustment. In Section~\ref{sect:mod}
we propose a slight modification to the one fifth rule, which is in the core of the self-adjustment strategy, in order to slow down the speed of
(dis)adaptation. We then conduct experiments on our set of benchmark problems in Section~\ref{sect:practice}
and confirm that the negative consequences of the misguidance imposed by the one fifth rule are somewhat damped. 
While this seems to retain the existence of the problems, it does improve the performance by at least a constant factor.
To prove that we did not break the goodness of the algorithm, we prove in Section~\ref{sect:theory} that the runtime on \textsc{OneMax} is still linear.
In Section~\ref{sect:best}, we investigate what would be the performance of the $(1+(\lambda,\lambda))$~GA if it always chose the best values for $\lambda$,
which we have experimentally observed in Section~\ref{sect:eui}, and the result is that there is still a room for improvement even in the hardest cases.
Finally, Section~\ref{sect:conclusion} concludes the paper.

\section{Preliminaries}\label{sect:pre}

We denote as $[1..n]$ the set of integer numbers $\{1, 2, \ldots, n\}$. The rest of the section
is dedicated to the definitions of the analyzed problems and evolutionary algorithms.

\subsection{Analyzed Problems}

Among other problems, we will consider \textsc{OneMax}, which is a classic test problem used in evolutionary computation.
The problem instances are functions $\textsc{OneMax}_{n,z}$, $z \in \{0,1\}^n$, defined on
bit strings of size $n$ as
\begin{equation*}
    \textsc{OneMax}_{n,z}: \{0,1\}^n \to \mathbf{R}; x \mapsto |\{i \in [1..n] \mid x_i = z_i\}|,
\end{equation*}
that is, $\textsc{OneMax}_{n,z}$ counts the number of bit positions in which $x$ and $z$ agree.
For a fixed $n$, all functions $\textsc{OneMax}_{n,z}$ constitute the \textsc{OneMax} problem of size $n$.

\textsc{OneMax} can be represented as a pseudo-Boolean polynomial function with the maximum degree of a monomial equal to 1,
so it is a \emph{linear} function. Along with \textsc{OneMax} we will consider a wider family of linear functions, defined as follows:
\begin{equation*}
    \textsc{Linear}_{n,z,w}: \{0,1\}^n \to \mathbf{R}; x \mapsto \sum_{i \in [1..n]} w_i \cdot [x_i = z_i],
\end{equation*}
where $w$ is the vector of positive weights associated with bit positions. We consider the cases when
the elements of $w$ are integer and are chosen uniformly at random from the set $[1..W]$, and call the corresponding problem
$\textsc{LinInt}_W$. Note that $\textsc{OneMax} = \textsc{LinInt}_1$.

We will also consider the special case of the MAX-3SAT problem~\cite{garey-johnson}, which aims at finding an assignment of
$n$ Boolean variables that maximizes the number of satisfied clauses in a Boolean formula written in the conjunctive normal form
with at most three variables (or their inversions) in each of $m$ clauses. While in general this problem is very hard,
some of the formulations are easier~\cite{sat-hard-and-easy}. We will use randomly generated formulas following the so-called \emph{planted solution} model.
In this model, there is a target assignment $x^*$ and the formula is chosen uniformly at random among all formulas on $n$ variables and with $m$ clauses
which are satisfied by $x^*$. This problem was studied in the context of theory of evolutionary computation~\cite{SuttonN14,doerr-neumann-sutton-oneplusone-cnf,buzdalovD-gecco17-3cnf}
and found to be similar to \textsc{OneMax} when the formula is \emph{dense}, that is $m \ge n^2$, and less similar but still tractable on average
when $m \ge n \log n$.

\subsection{The $\onell$ GA}

The $\onell$ Genetic Algorithm, or the $\onell$~GA for short, was proposed in~\cite{learning-from-black-box-thcs}.
Its main working principles are (i) to use mutation with a higher-than-usual mutation rate to speed up exploration
and (ii) crossover with the parent to diminish the destructive effects of this mutation.

The original paper~\cite{learning-from-black-box-thcs} mostly analyzed the version with the fixed parameter
$\lambda$, and it was proved that choosing a certain value for $\lambda$ as a function of problem size $n$
brings an $o(n \log n)$ runtime on \textsc{OneMax}, which is asymptotically faster than any other evolutionary algorithms at that time.
This was subsequently refined in~\cite{doerr-doerr-lambda-lambda-fixed-tight,doerr-doerr-lambda-lambda-parameters-journal}, where these bounds were further refined.

In the same time, \cite{learning-from-black-box-thcs} proved that setting $\lambda$ dynamically as a function not only of the problem size $n$,
but also of the current fitness, yields an $O(n)$ runtime on \textsc{OneMax}
and showed experimentally that a simple one fifth rule is able to keep $\lambda$ in the vicinity of the right value.
The fact that the runtime is also linear in this case was proven in~\cite{doerr-doerr-lambda-lambda-self-adjustment,doerr-doerr-lambda-lambda-parameters-journal}.
The linear runtime was also proven for MAX-SAT with large enough density of clauses~\cite{buzdalovD-gecco17-3cnf},
although it was found that, in order to be successful, the parameter $\lambda$ shall be kept hard under a sublinear bound
(the value of $\lambdabound = 2 \log n$ was used in that paper).

\begin{algorithm}[t]
\caption{$\onell$~GA with self-adjusting $\lambda \le \lambdabound$}\label{algo:ll-adjusting}
\begin{algorithmic}[1]
\State{$F \gets \text{constant} \in (1; 2)$} \Comment{Update strength}
\State{$U \gets 5$} \Comment{The 5 from the ``1/5-th rule''}
\State{$x \gets \Call{UniformRandom}{\{0, 1\}^n}$}
\For{$t \gets 1, 2, 3, \ldots$}
    \State{$p \gets \lambda / n$, $c \gets 1 / \lambda$, $\lambda' \gets [\lambda]$, $\ell \sim \mathcal{B}(n, p)$}
    \For{$i \in [1..\lambda']$} \Comment{Phase 1: Mutation}
        \State{$x^{(i)} \gets \Call{Mutate}{x, \ell}$}
    \EndFor
    \State{$x' \gets \Call{UniformRandom}{\{x^{(j)} \mid f(x^{(j)}) = \max\{f(x^{(i)})\}\}}$}
    \For{$i \in [1..\lambda']$} \Comment{Phase 2: Crossover}
        \State{$y^{(i)} \gets \Call{Crossover}{x, x', c}$}
    \EndFor
    \State{$y \gets \Call{UniformRandom}{\{y^{(j)} \mid f(y^{(j)}) = \max\{f(y^{(i)})\}\}}$}
    \If{$f(y) > f(x)$} \Comment{Selection and Adaptation}
        \State{$x \gets y$, $\lambda \gets \max\{\lambda / F, 1\}$}
    \ElsIf{$f(y) = f(x)$}
        \State{$x \gets y$, $\lambda \gets \min\{\lambda F^{1/(U - 1)}, \lambdabound\}$}
    \Else
        \State{$\lambda \gets \min\{\lambda F^{1/(U - 1)}, \lambdabound\}$}
    \EndIf
\EndFor
\end{algorithmic}
\end{algorithm}

The adaptive version of the $\onell$~GA, already with the bound on $\lambda$, is presented on Algorithm~\ref{algo:ll-adjusting}.
This algorithm uses the following operators to work on the bit-string individuals:
\begin{itemize}
    \item An $\ell$-bit mutation: The unary mutation operator $\Call{Mutate}{x, \ell}$ creates from $x \in \{0,1\}^n$ 
          a new bit string $y$ by flipping exactly $\ell$ bits chosen randomly without replacement.
    \item A ``biased'' uniform crossover: The binary operator $\Call{Crossover}{x, x', c}$ with the \emph{crossover bias} 
          $c \in [0,1]$ constructs a new bit string $y$ from two given bit strings $x$ and $x'$ by choosing for each $i \in [1..n]$ 
          the second argument's value ($y_i = x'_i$) with probability $c$ and setting $y_i = x_i$ otherwise.
\end{itemize}

We fix the mutation rate to be $p = \lambda / n$ and the crossover bias to be $c = 1 / \lambda$ as
and justified in~\cite{learning-from-black-box-thcs,Doerr16}.
After randomly initializing the one-element parent population $\{x\}$, in each iteration the following steps are performed:
\begin{itemize}
    \item In the \emph{mutation phase}, $\lambda$ offspring are sampled from the parent $x$ by applying the mutation for $\lambda$ times 
          where the step size $\ell$ is the same for all the offspring and is chosen at random from the binomial distribution $\mathcal{B}(n, p)$.
          Following the recently popularized ``more practice-aware'' paradigm~\cite{practice-aware},
          which has already given some remarkable results~\cite{pinto-doerr-crossover-ppsn18}, we resample $\ell$ if it was found to be zero.
    \item In an \emph{intermediate selection} step, the mutation offspring with maximal fitness, 
          called the \emph{mutation winner} and denoted by $x'$, is determined.
    \item In the \emph{crossover phase}, $\lambda$ offspring are created from $x$ and $x'$ using the crossover operator.
          To be ``practice-aware'' when performing crossovers, if a particular crossover requires borrowing none bits from $x'$,
          we repeat the sampling process again, and if it requires all bits to be sampled from $x'$,
          we do not perform fitness evaluation of the crossover's result.
    \item \emph{Elitist selection}. The best of the crossover offspring (breaking ties randomly and ignoring individuals identical to $x$) 
          replaces $x$ if its fitness is at least as large as the one of $x$.
\end{itemize}

The self-adjustment procedure of the $\onell$~GA is as follows. If the fitness is increased after an iteration,
then the value of $\lambda$ is reduced by a constant factor $F > 1$. 
If an iteration did not produce a fitness improvement, then $\lambda$ is increased by a factor of $F^{1/4}$. 
As a result, after a series of iterations with an average success rate of $1/5$, the algorithm ends up with the initial value of $\lambda$.
We bound $\lambda$ from below by one and from the above by some threshold $\lambdabound$, which typically depends on $n$ and is at most $n$.

\section{On Evaluations Until Improvement}\label{sect:eui}

\newcommand{\funnyfig}[1]{\scalebox{0.78}{
\begin{tikzpicture}
\begin{axis}[view={20}{10}, xlabel={Distance to the optimum}, ylabel={$\lambda$}, zlabel={Evaluations until improvement}, xmode=log, ymode=log, zmode=log]
\addplot3[surf] coordinates {
    (1,1,1)   (500,1,1)

    (1,50,1)  (500,50,1)
};
\addplot3+[only marks, mark size=1pt] table {plots/pgfm-#1.txt};
\addplot3[surf] file {plots/pgf-#1.txt};
\end{axis}
\end{tikzpicture}
}}

\begin{figure}[!]
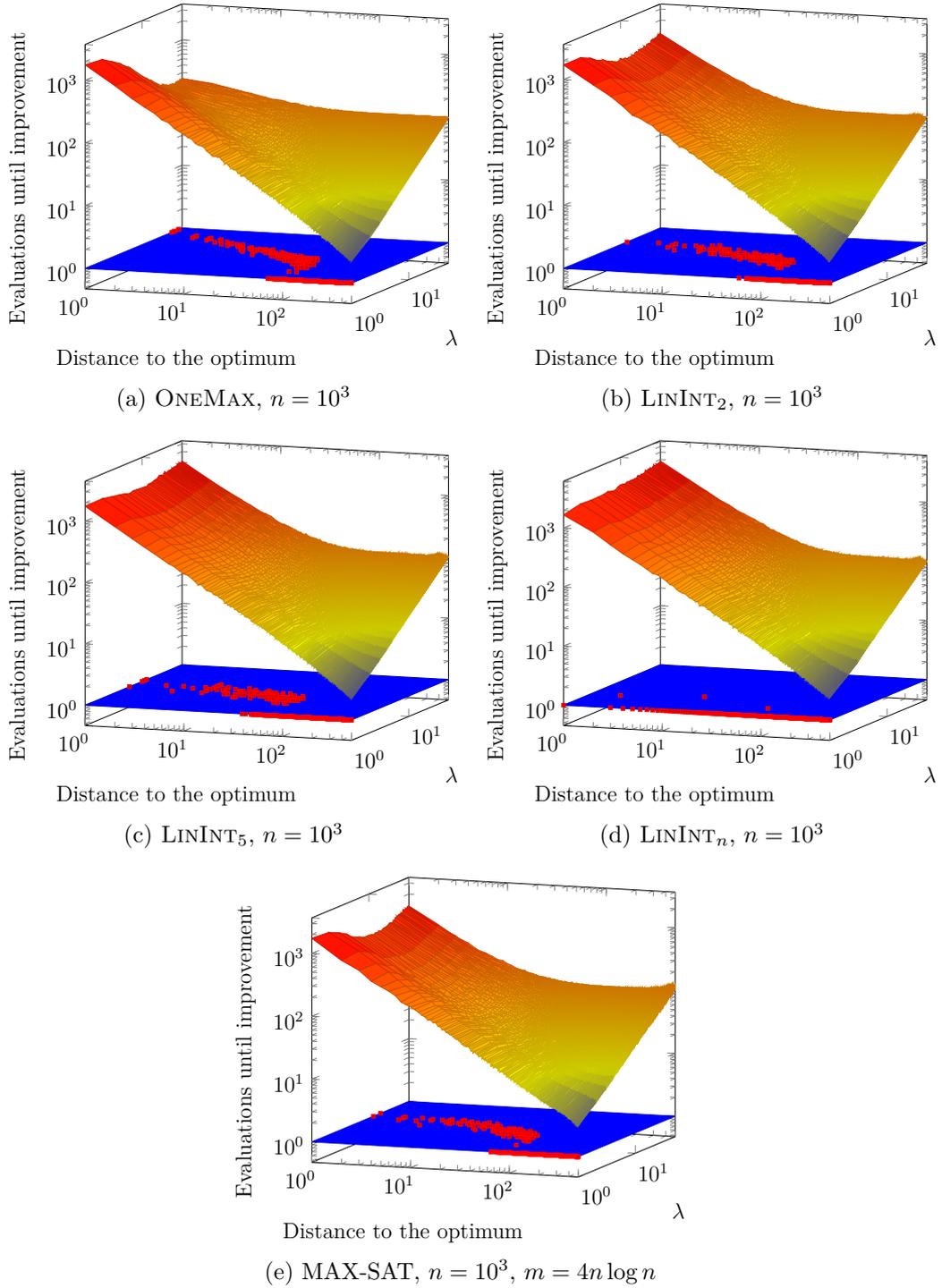

\centering
\subfloat[][\textsc{OneMax}, $n=10^3$]{\funnyfig{1000-1}\label{pp-1}}
\subfloat[][$\textsc{LinInt}_2$, $n=10^3$]{\funnyfig{1000-2}\label{pp-2}}\par
\subfloat[][$\textsc{LinInt}_5$, $n=10^3$]{\funnyfig{1000-5}\label{pp-5}}
\subfloat[][$\textsc{LinInt}_n$, $n=10^3$]{\funnyfig{1000-n}\label{pp-n}}\par
\subfloat[][MAX-SAT, $n=10^3$, $m=4n \log n$]{\funnyfig{1000-sat}\label{pp-sat}}
\caption{Performance plots for different problems. Red cubes are choices for $\lambda(d)$ within 2\% of the best.}
\end{figure}

In Fig.~\ref{pp-1}--\ref{pp-sat} we measured experimentally the impact of choosing particular values for $\lambda$,
depending on the Hamming distance to the optimum, for problems $\textsc{LinInt}_2$, $\textsc{LinInt}_5$, $\textsc{LinInt}_n$
and MAX-SAT with logarithmic clause density. We also repeat it for \textsc{OneMax}. For a single problem size $n = 10^3$,
and for all $\lambda_0 \in [1..50]$ we ran the $(1+(\lambda,\lambda))$ GA with the fixed $\lambda = \lambda_0$,
for $10^3$ times for each $\lambda_0$ and for each problem. During every run, we recorded, for all Hamming distances to the optimum
$1 \le d \le 500$, the total number of evaluations $E(d, \lambda)$ spent in this location, and the total number of events $I(d, \lambda)$ that the algorithm
leaves this location towards a smaller Hamming distance.

The top surfaces of the figures display an approximation of the \emph{expected number of evaluations until improvement}, computed as $E(d, \lambda) / I(d, \lambda)$.
To visualize the near-optimal choices of $\lambda$, we collected for every $d$ the values of $\lambda$ which lie within 2\% of the experimentally determined optimal choice
and displayed them as red cubes on the bottom surface of each plot. The blue bottom planes are drawn in order to deliver a better impression on where the cubes are actually located
regarding the coordinate axes.

The resulting look of Fig.~\ref{pp-1} supports the already known facts about the behaviour of the $(1+(\lambda,\lambda))$~GA on \textsc{OneMax}. Apart from the discretization
issues at larger distances from the optimum, good values for $\lambda$ are laid out along the line (in the logarithmic axes) originated at $d = 500, \lambda = 1$
and targeting towards $d = 1, \lambda \approx 30\ldots40$. By this it generally confirms once again that $\lambda \approx \sqrt{n/(n - f(x))}$ is a near-optimal choice for \textsc{OneMax}.

The results for $\textsc{LinInt}_2$ (Fig.~\ref{pp-2}) and $\textsc{LinInt}_5$ (Fig.~\ref{pp-5}), chosen for a reason that they ideally shall not deviate from \textsc{OneMax} by too much,
show that the picture is a little bit different. While the behaviour at distance of at most $d \approx 50$ seems generally similar to the one on \textsc{OneMax},
it changes on getting closer to the optimum. It appears visually that the trend is probably still linear in logarithmic axes, but the quotient is different
and at $d \approx 1$ the best choices of $\lambda$ for $\textsc{LinInt}_2$ are slightly above 10, and for $\textsc{LinInt}_5$ they are slightly below 10.
The top surfaces are even more different: the best observed expected progress is around $10^3$ for $\textsc{LinInt}_2$ and comes closer to $2\cdot10^3$ for $\textsc{LinInt}_5$,
as opposed to $\approx200$ for \textsc{OneMax}. This might mean that the progress is too slow for the one fifth rule to keep $\lambda$ in a good shape.

The extreme example of $\textsc{LinInt}_n$ (Fig.~\ref{pp-n}) shows that when this correlation finally breaks, the sole structure of the $(1+(\lambda,\lambda))$~GA is not sufficient in its current
form to optimize this problem well. The optimal values of $\lambda$ are concentrated around $\lambda = 1$, and the entire landscape is ``too hill-climber-friendly'' from the 
perspective of the $(1+(\lambda,\lambda))$~GA.

An interesting behaviour is demonstrated on the MAX-SAT problem with logarithmic density of clauses (Fig.~\ref{pp-sat}). The landscape seems to be somewhere between
the ones for \textsc{OneMax} and for $\textsc{LinInt}_2$, while the curve of optimal $\lambda$ appears to be bent, as if for large distances to the optimum the problem is just like 
\textsc{OneMax}, but gets disproportionally more complicated towards the optimum. The positive effect from bounding $\lambda \le 2 \log n$, as in~\cite{buzdalovD-gecco17-3cnf},
maps well to this picture.

\section{Proposed Modification}\label{sect:mod}

The proposed modification to the self-adjustment rule of the $(1+(\lambda,\lambda))$~GA is outlined on Algorithm~\ref{algo:ll-mod}.

The main idea is to prohibit the immediate growth of $\lambda$ on long unsuccessful runs, while allowing raising it arbitrarily high in more steps if really needed.
On a successful iteration, the modified value of $\lambda$ is additionally saved in a variable $\lambda_0$.
Any sequence of unsuccessful iterations is now split in \emph{spans} of linearly increasing lengths, starting with 10 to be more like the original $(1+(\lambda,\lambda))$~GA
in the case the adaptation works well. The length is incremented by 1 once the entire span is used up, and the next span starts again with $\lambda_0$. By this strategy, we limit
the speed with which the maximum value of $\lambda$ grows, and in the same time we retain the chances to perform iteration with rather small $\lambda$, which may be of use
when small $\lambda$ are better.

This modification roughly squares the number of iterations, needed by the $(1+(\lambda,\lambda))$~GA to reach a certain distant value of $\lambda$.
As a result, when the maximal $\lambda$ overshoots the optimal value for the current distance, the algorithm still has some iterations to spend around the optimal values
even if there is no fitness improvement yet, unlike the original $(1+(\lambda,\lambda))$~GA.

We still reserve the right to bound $\lambda$ from above by a given $\lambdabound$. In our experiments, we consider both cases, $\lambdabound = n$ and $\lambdabound = 2 \log n$,
to assess the impact of the proposed strategy even in the constrained case, while keeping track of what happens when the strategy works on its own.

\begin{algorithm}[!t]
	\caption{$\onell$~GA, \newalg{modified} self-adjustment of $\lambda \le \lambdabound$}\label{algo:ll-mod}
	\begin{algorithmic}[1]
		\State{$F \gets \text{constant} \in (1; 2)$} \Comment{Update strength}
		\State{$U \gets 5$} \Comment{The 5 from the ``1/5-th rule''}
		\newalg{\State{$B \gets 0$} \Comment{Bad iterations in a row}}
		\newalg{\State{$\Delta \gets 10$} \Comment{Span for growing $\lambda$}}
		\newalg{\State{$\lambda_{+} \gets 1$} \Comment{The base for $\lambda$}}
		\State{$x \gets \Call{UniformRandom}{\{0, 1\}^n}$}
		\For{$t \gets 1, 2, 3, \ldots$}
		\State{$p \gets \lambda / n$, $c \gets 1 / \lambda$, $\lambda' \gets [\lambda]$, $\ell \sim \mathcal{B}(n, p)$}
		\For{$i \in [1..\lambda']$} \Comment{Phase 1: Mutation}
		\State{$x^{(i)} \gets \Call{Mutate}{x, \ell}$}
		\EndFor
		\State{$x' \gets \Call{UniformRandom}{\{x^{(j)} \mid f(x^{(j)}) = \max\{f(x^{(i)})\}\}}$}
		\For{$i \in [1..\lambda']$} \Comment{Phase 2: Crossover}
		\State{$y^{(i)} \gets \Call{Crossover}{x, x', c}$}
		\EndFor
		\State{$y \gets \Call{UniformRandom}{\{y^{(j)} \mid f(y^{(j)}) = \max\{f(y^{(i)})\}\}}$}
		\If{$f(y) > f(x)$} \Comment{Selection and Adaptation}
		\State{$x \gets y$, $\lambda \gets \max\{\lambda / F, 1\}$, \newalg{$\lambda_0 \gets \lambda$, $B \gets 0$, $\Delta \gets 10$}}
		\Else
		\If{$f(y) = f(x)$}
		\State{$x \gets y$}
		\EndIf
		\newalg{\State{$B \gets B + 1$}
			\If{$B = \Delta$}
			\State{$B \gets 0$, $\Delta \gets \Delta + 1$}
			\EndIf}
		\State{$\lambda \gets \min\{\newalg{\lambda_0} F^{\newalg{B}/(U - 1)}, \lambdabound\}$}
		\EndIf
		\EndFor
	\end{algorithmic}
\end{algorithm}

\section{Experimental Evaluation}\label{sect:practice}

We have evaluated the original $(1+(\lambda,\lambda))$~GA in two variations ($\lambdabound = n$ and $\lambdabound = 2 \log n$),
the same algorithm with the modified adaptation (again in these two variations), as well as the $(1+1)$~EA with the standard bit mutation
(using resampling when zero bits appear to be flipped) and the randomized local search.
The same five problems are used in experiments (\textsc{OneMax}, $\textsc{LinInt}_2$, $\textsc{LinInt}_5$, $\textsc{LinInt}_n$, MAX-SAT).
We ran each algorithm for 100 times on each problem with $n \in \{100, 200, 400, 800, 1600, 3200, 6400, 12800 \}$.

The results are presented in Fig.~\ref{rt-1}--\ref{rt-sat}. On \textsc{OneMax}, all variants of $(1+(\lambda,\lambda))$~GA behave well as expected (Fig.~\ref{rt-1}).
In particular, the logarithmic versions are slightly inferior to the unlimited versions, which is also expected as the former use suboptimal $\lambda$ values at the
final iterations. The proposed algorithm is also slightly inferior to the original one, when the matching versions are compared, but the difference is small. It appears that
it has a constant-multiple overhead of some 10\%, however, the dynamic is still linear.

\newcommand{\plotscaleout}{0.6}
\newcommand{\plotscalein}{0.8\textwidth}

\begin{figure}[!t]
\centering
\subfloat[][\textsc{OneMax}]{
\scalebox{\plotscaleout}{\begin{tikzpicture}
\begin{axis}[width=0.8\textwidth, enlargelimits=false, xmode=log, log basis x = 2, xlabel={Problem size}, ylabel={Evaluations / $x$}, legend pos=north west]
  \addplot plot[error bars/.cd, y dir=both, y explicit] coordinates{(100,5.5637)+-(0,1.0288)(200,5.92125)+-(0,0.83265)(400,6.23755)+-(0,0.577925)(800,6.3702875)+-(0,0.42737499999999995)(1600,6.46545625)+-(0,0.30089375)(3200,6.553309375)+-(0,0.274425)(6400,6.599128125)+-(0,0.190840625)(12800,6.62523671875)+-(0,0.15856015625)};
  \addlegendentry{\small$(1+(\lambda,\lambda)), n$};
  \addplot plot[error bars/.cd, y dir=both, y explicit] coordinates{(100,5.7714)+-(0,0.9325)(200,6.0132)+-(0,0.8466)(400,6.238474999999999)+-(0,0.7141500000000001)(800,6.5728375)+-(0,0.5592625)(1600,6.734618749999999)+-(0,0.46436875)(3200,6.8443)+-(0,0.38910312500000005)(6400,6.9554640625)+-(0,0.3770046875)(12800,7.08593203125)+-(0,0.31499921875)};
  \addlegendentry{\small$(1+(\lambda,\lambda)), \log n$};
  \addplot plot[error bars/.cd, y dir=both, y explicit] coordinates{(100,4.4351)+-(0,1.0248)(200,5.1535)+-(0,1.4615)(400,5.7822249999999995)+-(0,1.20115)(800,6.4304625)+-(0,1.205125)(1600,7.07109375)+-(0,1.10829375)(3200,7.91875)+-(0,1.298853125)(6400,8.607821874999999)+-(0,1.3050296875)(12800,9.240272656250001)+-(0,1.3602125)};
  \addlegendentry{\small RLS};
  \addplot plot[error bars/.cd, y dir=both, y explicit] coordinates{(100,7.0224)+-(0,2.5888999999999998)(200,8.3438)+-(0,2.1153)(400,9.03625)+-(0,1.962475)(800,9.965925)+-(0,2.3120875)(1600,11.536175)+-(0,2.0443625)(3200,12.440859375)+-(0,2.13119375)(6400,13.813815625)+-(0,2.0358984375)(12800,14.8797796875)+-(0,1.935478125)};
  \addlegendentry{\small(1+1) EA};
  \addplot plot[error bars/.cd, y dir=both, y explicit] coordinates{(100,6.3083)+-(0,1.1414)(200,6.489349999999999)+-(0,1.01475)(400,6.54845)+-(0,0.7827500000000001)(800,6.8065125)+-(0,0.6276125)(1600,6.88228125)+-(0,0.4787625)(3200,6.9262875)+-(0,0.38632500000000003)(6400,6.955978125000001)+-(0,0.30210000000000004)(12800,6.971492187500001)+-(0,0.18661562499999998)};
  \addlegendentry{\small$(1+(\lambda,\lambda)), n*$};
  \addplot plot[error bars/.cd, y dir=both, y explicit] coordinates{(100,6.3162)+-(0,1.2541)(200,6.666900000000001)+-(0,1.0293999999999999)(400,6.80655)+-(0,0.77705)(800,6.984475)+-(0,0.6951)(1600,7.086225)+-(0,0.64686875)(3200,7.14023125)+-(0,0.448865625)(6400,7.281140625000001)+-(0,0.4125)(12800,7.3771070312500004)+-(0,0.35529921875)};
  \addlegendentry{\small$(1+(\lambda,\lambda)), \log n*$};
\end{axis}
\end{tikzpicture}}\label{rt-1}}
\subfloat[][$\textsc{LinInt}_2$]{
\scalebox{\plotscaleout}{\begin{tikzpicture}
\begin{axis}[width=\plotscalein, enlargelimits=false, xmode=log, log basis x = 2, xlabel=Problem size, ylabel=Evaluations / $x$, legend pos=north west]
  \addplot plot[error bars/.cd, y dir=both, y explicit] coordinates{(100,6.0763)+-(0,1.204)(200,6.90485)+-(0,1.2181)(400,8.109525)+-(0,1.9867750000000002)(800,8.565137499999999)+-(0,1.8231625)(1600,10.4639375)+-(0,3.543675)(3200,12.708878125000002)+-(0,3.960221875)(6400,15.357650000000001)+-(0,4.4171625)(12800,17.57313984375)+-(0,3.879203125)};
  \addlegendentry{\small$(1+(\lambda,\lambda)), n$};
  \addplot plot[error bars/.cd, y dir=both, y explicit] coordinates{(100,6.1152999999999995)+-(0,1.3293000000000001)(200,6.707050000000001)+-(0,0.9616)(400,7.450325)+-(0,1.05745)(800,7.843175)+-(0,1.09345)(1600,8.454475)+-(0,1.12859375)(3200,9.145178125)+-(0,1.1332187500000002)(6400,9.748978125)+-(0,1.0820640625)(12800,10.353447656250001)+-(0,1.024196875)};
  \addlegendentry{\small$(1+(\lambda,\lambda)), \log n$};
  \addplot plot[error bars/.cd, y dir=both, y explicit] coordinates{(100,4.3529)+-(0,1.0293999999999999)(200,5.0988)+-(0,1.3125499999999999)(400,6.0539250000000004)+-(0,1.4759)(800,6.4114875)+-(0,1.2976750000000001)(1600,7.1292)+-(0,1.10215)(3200,7.928203125)+-(0,1.320284375)(6400,8.625865625)+-(0,1.30164375)(12800,9.39326953125)+-(0,1.26112265625)};
  \addlegendentry{\small RLS};
  \addplot plot[error bars/.cd, y dir=both, y explicit] coordinates{(100,6.9714)+-(0,2.2975)(200,8.18665)+-(0,2.43105)(400,9.285475)+-(0,2.442975)(800,10.2239625)+-(0,1.9377250000000001)(1600,11.52631875)+-(0,2.462325)(3200,12.589859375000001)+-(0,2.1131312500000003)(6400,13.9055171875)+-(0,1.990340625)(12800,15.213092968749999)+-(0,2.2511085937499997)};
  \addlegendentry{\small(1+1) EA};
  \addplot plot[error bars/.cd, y dir=both, y explicit] coordinates{(100,6.5866999999999996)+-(0,1.3476)(200,6.972300000000001)+-(0,1.21535)(400,7.612825)+-(0,1.134575)(800,8.3706625)+-(0,1.2951374999999998)(1600,9.024099999999999)+-(0,1.54860625)(3200,9.94128125)+-(0,1.7436656249999998)(6400,11.163350000000001)+-(0,2.029009375)(12800,13.19833984375)+-(0,2.9658125)};
  \addlegendentry{\small$(1+(\lambda,\lambda)), n*$};
  \addplot plot[error bars/.cd, y dir=both, y explicit] coordinates{(100,6.5587)+-(0,1.5301)(200,6.8761)+-(0,1.2759)(400,7.612424999999999)+-(0,1.097475)(800,8.1857875)+-(0,1.2109874999999999)(1600,8.61686875)+-(0,1.03825625)(3200,9.084521875)+-(0,0.9670531250000001)(6400,9.764753125)+-(0,0.9554515625000001)(12800,10.538949218749998)+-(0,1.31181484375)};
  \addlegendentry{\small$(1+(\lambda,\lambda)), \log n*$};
\end{axis}
\end{tikzpicture}}\label{rt-2}}
\par
\subfloat[][$\textsc{LinInt}_5$]{
\scalebox{\plotscaleout}{\begin{tikzpicture}
\begin{axis}[width=\plotscalein, enlargelimits=false, xmode=log, log basis x = 2, xlabel=Problem size, ylabel=Evaluations / $x$, legend pos=north west]
  \addplot plot[error bars/.cd, y dir=both, y explicit] coordinates{(100,7.2784)+-(0,2.5418000000000003)(200,8.31785)+-(0,2.2033)(400,10.12525)+-(0,2.9657)(800,12.0535125)+-(0,3.5171125)(1600,15.151943750000001)+-(0,4.61335)(3200,17.58734375)+-(0,5.552490625000001)(6400,20.181012499999998)+-(0,4.2680609375)(12800,22.5550859375)+-(0,4.899)};
  \addlegendentry{\small$(1+(\lambda,\lambda)), n$};
  \addplot plot[error bars/.cd, y dir=both, y explicit] coordinates{(100,7.2938)+-(0,2.1328)(200,7.95025)+-(0,2.03555)(400,8.9665)+-(0,1.6615)(800,9.837150000000001)+-(0,1.6979750000000002)(1600,10.964281249999999)+-(0,1.9637)(3200,12.147490625)+-(0,1.7185875)(6400,13.773721875000001)+-(0,2.4058281249999998)(12800,15.184000781249999)+-(0,2.18509453125)};
  \addlegendentry{\small$(1+(\lambda,\lambda)), \log n$};
  \addplot plot[error bars/.cd, y dir=both, y explicit] coordinates{(100,4.5963)+-(0,1.4453)(200,5.4157)+-(0,1.3307499999999999)(400,5.781549999999999)+-(0,1.135525)(800,6.6718875)+-(0,1.5083375)(1600,7.38183125)+-(0,1.38455625)(3200,7.8161000000000005)+-(0,1.241959375)(6400,8.533551562500001)+-(0,1.1649890625000001)(12800,9.119334375)+-(0,1.192096875)};
  \addlegendentry{\small RLS};
  \addplot plot[error bars/.cd, y dir=both, y explicit] coordinates{(100,7.0297)+-(0,2.1848)(200,7.98955)+-(0,1.9853999999999998)(400,9.2513)+-(0,1.9627000000000001)(800,10.472987499999999)+-(0,2.040075)(1600,11.675418749999999)+-(0,2.1023187500000002)(3200,12.78645625)+-(0,2.462084375)(6400,14.03050625)+-(0,2.1053484375)(12800,15.16709453125)+-(0,2.20628359375)};
  \addlegendentry{\small(1+1) EA};
  \addplot plot[error bars/.cd, y dir=both, y explicit] coordinates{(100,7.277200000000001)+-(0,2.2181)(200,8.52355)+-(0,1.89125)(400,9.032175)+-(0,1.825275)(800,10.13495)+-(0,2.04405)(1600,12.043412499999999)+-(0,3.05435)(3200,14.33079375)+-(0,3.34886875)(6400,16.8719703125)+-(0,4.3126078125)(12800,19.14669453125)+-(0,4.2773101562499996)};
  \addlegendentry{\small$(1+(\lambda,\lambda)), n*$};
  \addplot plot[error bars/.cd, y dir=both, y explicit] coordinates{(100,7.271599999999999)+-(0,2.1923)(200,7.963550000000001)+-(0,1.50575)(400,9.412975)+-(0,1.74795)(800,10.145275)+-(0,1.7477)(1600,10.96315)+-(0,1.72420625)(3200,12.034090625)+-(0,1.666096875)(6400,13.6460390625)+-(0,1.9782546875)(12800,14.9890453125)+-(0,2.187259375)};
  \addlegendentry{\small$(1+(\lambda,\lambda)), \log n*$};
\end{axis}
\end{tikzpicture}}\label{rt-5}}
\subfloat[][$\textsc{LinInt}_n$]{
\scalebox{\plotscaleout}{\begin{tikzpicture}
\begin{axis}[width=\plotscalein, enlargelimits=false, xmode=log, log basis x = 2, xlabel=Problem size, ylabel=Evaluations / $x$, legend pos=north west]
  \addplot plot[error bars/.cd, y dir=both, y explicit] coordinates{(100,8.4499)+-(0,2.8956)(200,9.4912)+-(0,2.6322)(400,12.502650000000001)+-(0,4.111125)(800,14.0776875)+-(0,3.744475)(1600,17.2562)+-(0,4.78709375)(3200,20.64246875)+-(0,4.75829375)(6400,23.3719296875)+-(0,4.3533671875)(12800,26.6823578125)+-(0,6.0413515625)};
  \addlegendentry{\small$(1+(\lambda,\lambda)), n$};
  \addplot plot[error bars/.cd, y dir=both, y explicit] coordinates{(100,8.232000000000001)+-(0,2.5228)(200,9.624550000000001)+-(0,2.6355)(400,11.446575000000001)+-(0,3.0713)(800,12.428525)+-(0,2.9115249999999997)(1600,13.39275625)+-(0,2.53851875)(3200,15.738050000000001)+-(0,2.5584375)(6400,16.6544953125)+-(0,2.5982421875)(12800,18.46009765625)+-(0,2.969246875)};
  \addlegendentry{\small$(1+(\lambda,\lambda)), \log n$};
  \addplot plot[error bars/.cd, y dir=both, y explicit] coordinates{(100,4.3662)+-(0,1.1856)(200,5.07945)+-(0,1.19615)(400,5.809375)+-(0,1.17075)(800,6.6073375)+-(0,1.2174625000000001)(1600,7.26019375)+-(0,1.40701875)(3200,7.7561968750000005)+-(0,1.19429375)(6400,8.757476562499999)+-(0,1.3852484375)(12800,9.32205)+-(0,1.1363796875)};
  \addlegendentry{\small RLS};
  \addplot plot[error bars/.cd, y dir=both, y explicit] coordinates{(100,6.9096)+-(0,2.0307)(200,8.377749999999999)+-(0,2.0513999999999997)(400,9.24295)+-(0,2.480975)(800,10.7948)+-(0,2.0857625)(1600,11.43244375)+-(0,1.926925)(3200,12.55955)+-(0,1.7748249999999999)(6400,14.31246875)+-(0,2.1232953125)(12800,15.375760156250001)+-(0,2.0111484375)};
  \addlegendentry{\small(1+1) EA};
  \addplot plot[error bars/.cd, y dir=both, y explicit] coordinates{(100,7.9952)+-(0,2.4788)(200,8.95435)+-(0,2.4393000000000002)(400,10.044425)+-(0,2.39575)(800,12.5736375)+-(0,2.700025)(1600,14.42355)+-(0,3.4830437499999998)(3200,16.741609375)+-(0,3.6047375)(6400,20.135346875)+-(0,4.4348671875)(12800,22.997716406250003)+-(0,3.75469140625)};
  \addlegendentry{\small$(1+(\lambda,\lambda)), n*$};
  \addplot plot[error bars/.cd, y dir=both, y explicit] coordinates{(100,7.5630999999999995)+-(0,2.1526)(200,9.42955)+-(0,2.6067)(400,10.615875)+-(0,2.62)(800,12.0263625)+-(0,3.0488875)(1600,13.59253125)+-(0,2.6892125)(3200,14.930578125)+-(0,2.912553125)(6400,16.7242546875)+-(0,2.8265625)(12800,18.484753125)+-(0,2.60756953125)};
  \addlegendentry{\small$(1+(\lambda,\lambda)), \log n*$};
\end{axis}
\end{tikzpicture}}\label{rt-n}}
\par
\subfloat[][MAX-SAT]{
\scalebox{\plotscaleout}{\begin{tikzpicture}
\begin{axis}[width=\plotscalein, enlargelimits=false, xmode=log, log basis x = 2, xlabel=Problem size, ylabel=Evaluations / $x$, legend pos=north west]
  \addplot plot[error bars/.cd, y dir=both, y explicit] coordinates{(100,8.1999)+-(0,2.0203)(200,8.82115)+-(0,2.1656)(400,8.79645)+-(0,2.5663)(800,9.8784125)+-(0,3.2777749999999997)(1600,10.46711875)+-(0,3.04811875)(3200,11.311596875)+-(0,3.478371875)(6400,13.3838546875)+-(0,3.2515828125)(12800,15.833101562500001)+-(0,4.6081171875)};
  \addlegendentry{\small$(1+(\lambda,\lambda)), n$};
  \addplot plot[error bars/.cd, y dir=both, y explicit] coordinates{(100,7.9216999999999995)+-(0,1.7783000000000002)(200,8.27005)+-(0,1.32195)(400,8.251325000000001)+-(0,1.14785)(800,8.3130875)+-(0,1.0241375)(1600,8.45760625)+-(0,0.92121875)(3200,8.719362499999999)+-(0,0.90501875)(6400,9.1062390625)+-(0,0.9041906249999999)(12800,9.121128125)+-(0,0.817896875)};
  \addlegendentry{\small$(1+(\lambda,\lambda)), \log n$};
  \addplot plot[error bars/.cd, y dir=both, y explicit] coordinates{(100,5.8765)+-(0,1.5708000000000002)(200,6.19015)+-(0,1.5131000000000001)(400,6.59885)+-(0,1.292875)(800,6.9539125)+-(0,1.209175)(1600,7.775393749999999)+-(0,1.1961875)(3200,8.270184375)+-(0,1.247415625)(6400,8.9110484375)+-(0,1.1626812499999999)(12800,9.52956953125)+-(0,1.31337578125)};
  \addlegendentry{\small RLS};
  \addplot plot[error bars/.cd, y dir=both, y explicit] coordinates{(100,8.093300000000001)+-(0,2.077)(200,9.1806)+-(0,2.1462)(400,9.8097)+-(0,2.02555)(800,10.788725)+-(0,1.9938125)(1600,11.82930625)+-(0,2.1046625)(3200,13.459503125000001)+-(0,1.8652406250000002)(6400,14.4594)+-(0,2.072440625)(12800,16.11074921875)+-(0,2.74694609375)};
  \addlegendentry{\small(1+1) EA};
  \addplot plot[error bars/.cd, y dir=both, y explicit] coordinates{(100,8.1039)+-(0,1.8659000000000001)(200,8.62015)+-(0,1.7155500000000001)(400,8.947474999999999)+-(0,1.5182)(800,8.7839375)+-(0,1.2811750000000002)(1600,8.98741875)+-(0,1.04271875)(3200,9.185315625)+-(0,1.102446875)(6400,9.8719046875)+-(0,1.3245125)(12800,11.031305468749999)+-(0,2.19859921875)};
  \addlegendentry{\small$(1+(\lambda,\lambda)), n*$};
  \addplot plot[error bars/.cd, y dir=both, y explicit] coordinates{(100,8.2801)+-(0,2.0283)(200,8.59475)+-(0,1.6172499999999999)(400,8.4786)+-(0,1.246675)(800,8.5647)+-(0,1.1354625)(1600,8.720575)+-(0,0.89471875)(3200,9.004346875)+-(0,0.9694625000000001)(6400,9.119115625)+-(0,0.8490421874999999)(12800,9.28373984375)+-(0,0.7835968750000001)};
  \addlegendentry{\small$(1+(\lambda,\lambda)), \log n*$};
\end{axis}
\end{tikzpicture}}\label{rt-sat}}
\caption{Runtimes on different functions}
\end{figure}
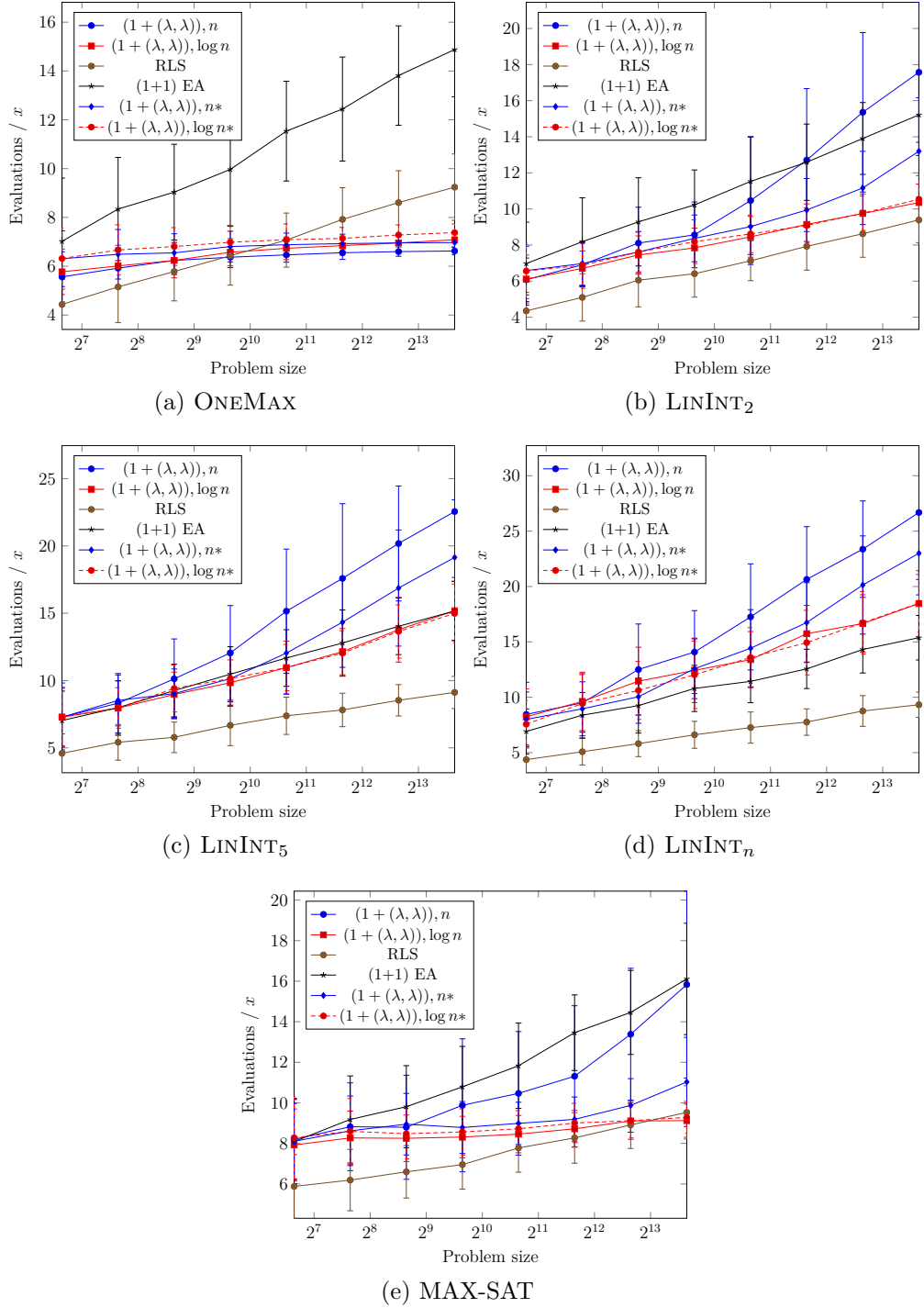

On $\textsc{LinInt}_2$ (Fig.~\ref{rt-2}) the logarithmically constrained versions still behave quite well, although not linearly. They are faster than the $(1+1)$~EA, and slightly
slower than RLS, and the general trend looks as if they will eventually hit and overcome RLS. The unconstrained versions, on the contrary, already behave noticeably worse.
The original unconstrained $(1+(\lambda,\lambda))$~GA slows down as a trend already at $n = 1600$ and gets worse than the $(1+1)$~EA at $n = 6400$. While at the latter problem size
the modified unconstrained $(1+(\lambda,\lambda))$~GA also starts to slow down, it is still below the $(1+1)$~EA, so at least the constant factor is smaller.

The situation is similar in Fig.~\ref{rt-5} with $\textsc{LinInt}_5$ and in Fig.~\ref{rt-n} with $\textsc{LinInt}_n$.
The general trend of the unconstrained algorithms is to rise above the runtime of the $(1+1)$~EA, but the one with the modified self-adjustment strategy is always better.
The slopes of these plots allows conjecturing that the runtime scales as $\Theta(n (\log n)^2)$, however, a proof or a more elaborate experiment is desirable.

The situation with the MAX-SAT problem (Fig.~\ref{rt-sat}) looks a little bit different, as the original unconstrained version seems to climb the plot at much higher rates
than the modified one. Whether this is an artifact of the scale, or the asymptotics of the runtime of these versions are really different, is impossible to completely 
derive from the plots. Unfortunately, MAX-SAT is slower than linear functions to either compute the fitness or to recompute it incrementally when a small number of bits change,
and also is harder to reason about from the theoretic point of view, so getting this part of the picture in order will take more time and resources.

\section{Runtime Analysis on \textsc{OneMax}}\label{sect:theory}

In this section we show that the suggested modification of the self-adjustment strategy for $\lambda$ retains the linear runtime on \textsc{OneMax}.
The proof of this result is based on theorems and claims from~\cite{doerr-doerr-lambda-lambda-self-adjustment}.
In particular, we retain the parameter tuning of $p = \lambda / n$ and $c = 1 / \lambda$.
We also retain the original adaptation speed factor $F$ of 1.5, which meets the proof restrictions in~\cite{doerr-doerr-lambda-lambda-self-adjustment}.
\begin{theorem}\label{main-runtime-theorem}
	The optimization time of the self-adjusting $(1+(\lambda,\lambda))$ GA
	with a capability of resetting $\lambda$ to the last successful value and parameters $p = \lambda / n$ and $c = 1 / \lambda$
	on every \textsc{OneMax} function is O(n) for sufficiently small update strength $F \in (1,2)$.
\end{theorem}

\begin{proof}
The strategy of the proof does not change compared to~\cite[Theorem 5]{doerr-doerr-lambda-lambda-self-adjustment}.
We show that the population size $\lambda$ still does not usually differ much from the optimal choice $\lambda^* = \lceil\sqrt{n/n-f(x)}\rceil$ 
analyzed in~\cite[Theorem 2]{doerr-doerr-lambda-lambda-self-adjustment}.

Let $x \in \{0,1\}^n$, $\lambda \geq C_0\lceil\sqrt{n/n-f(x)}\rceil$, and $q=q(\lambda)$ be the probability that one iteration of the GA starting with $x$ is successful.
By~\cite[Lemma 6]{doerr-doerr-lambda-lambda-self-adjustment}, $q > 1/5$ for all $C_0 > C$.
Thus, for cases when $\lambda$ is much larger than $\lambda^*$, such a reasonably high probability gives
an ability to assume that on one of the next iterations $\lambda$ value is reduced and tends to $\lambda^*$.   

We divide the optimization process into phases as suggested in~\cite{doerr-doerr-lambda-lambda-self-adjustment}.
The first type of phases is called the \emph{short} phase. By definition, this is a phase such that, during all GA iterations which constitute the phase,
an inequality $\lambda \leq C_0 \lambda^*$ stays true. Let $\tilde{\lambda}$ be the initial $\lambda$ value and let $t$ be the number of iterations in the phase. 
There are $2\lambda$ fitness evaluations on single GA iteration with parameter $\lambda$, thus the total number of fitness evaluations for the modified algorithm at this phase
is as follows. If $t \leq \tilde{d}$:
\begin{equation}\label{short-phase-upper-bound-1}
\sum_{i=0}^{t-1} 2\tilde{\lambda}F^{i/4} = 2\tilde{\lambda}\frac{F^{t/4}-1}{F^{1/4}-1} = O(\lambda^*),
\end{equation}
otherwise:
\begin{equation}\label{short-phase-upper-bound-2}
\sum_{d=\tilde{d}}^{k-1} \sum_{i=0}^{d-1} 2\tilde{\lambda}F^{i/4} + \sum_{j=0}^{c-1} 2\tilde{\lambda}F^{j/4} \leq \sum_{i=0}^{t-1} 2\tilde{\lambda}F^{i/4} = 2\tilde{\lambda}\frac{F^{t/4}-1}{F^{1/4}-1} = O(\lambda^*)
\end{equation}
\begin{align}
\sum_{d=\tilde{d}}^{k-1} \sum_{i=0}^{d-1} 2\tilde{\lambda}F^{i/4} + \sum_{j=0}^{c-1} 2\tilde{\lambda}F^{j/4} &= 
2\tilde{\lambda} \frac{\sum_{d=\tilde{d}}^{k-1}(F^{d/4}-1) + F^{c/4} - 1}{F^{1/4}-1}\notag\\
&\geq O(C_0\lambda^*) =O(\lambda^*), \label{short-phase-lower-bound}
\end{align}
where $(k - \tilde{d})$ denotes the number of resets of $\lambda$ to the last successful value, $c$ is the number of iterations
after the last $\lambda$ reset, $\sum_{i=0}^{k} (\tilde{d} + k) + c = t$ and $\tilde{d} = 10$.

The second type of optimization phases is the \emph{long} phase.
In these phases, $\lambda$ is equal or exceeds $C_0 \lambda^*$ for at least one iteration. In turn, the \emph{long} phase splits into an \emph{opening}
sub-phase, which ends when the last iteration with $\lambda < C_0 \lambda^*$ occurs, and the \emph{main} sub-phase, which starts with population size 
of $C_0 \lambda^* \leq \lambda < C_0 \lambda^* F^{1/4}$.

For an \emph{opening} sub-phase, the number of fitness evaluations can be calculated similar to~(\ref{short-phase-upper-bound-1}),~(\ref{short-phase-upper-bound-2}) and~(\ref{short-phase-lower-bound}).
Making a correction that the number of iterations $t = m$, where $m = \max(k): \lambda F^{k/4} < C_0 \lambda^*$, we show that the sum of fitness evaluations is $O(\lambda^*)$. In a same manner, the cost
of the \emph{main} phase can be evaluated, having the initial $\lambda$ at most $C_0 \lambda^* F^{1/4}$.

Consider case when $t > \tilde{d}$, which is different from original one fifth success rule adjustment:
\begin{align}
C_0 \lambda^* F^{1/4} \left(\sum_{d=\tilde{d}}^{k-1} \sum_{i=0}^{d} F^{i/4} + \sum_{j=0}^{c} F^{j/4}\right) &\leq C_0 \lambda^* F^{1/4} \sum_{i=1}^{t} F^{i/4}\notag\\
 &\leq D' C_0 \lambda^* F^{(t+1)/4}\label{long-phase-expectation}
\end{align}
for $D' \geq  1/ (F^{1/4} - 1)$ where $(k - \tilde{d})$ denotes the number of resets of $\lambda$ to the last successful value, $c$ equals to the iterations number after the last $\lambda$ reset,
$\sum_{i=0}^{k} (\tilde{d} + k) + c = t$ and $\tilde{d} = 10$.

Equation (\ref{long-phase-expectation}) proves the equivalent of~\cite[Claim 2.1]{doerr-doerr-lambda-lambda-self-adjustment}:
$E[T \mid I = t] \leq D \lambda^* F^{t/4}$ for large enough constant $D$, where $T$ is the number of fitness evaluations during the \emph{long} phase, and $I$ is the number of iterations in the 
\emph{main} sub-phase.

The statement~\cite[Claim 2.2]{doerr-doerr-lambda-lambda-self-adjustment}, which provides the value of probability that the \emph{main} sub-phase
requires $t$ iterations $\Pr[I = t] = e^{-ct}, c > 0$, is still valid for the modified algorithm, because, generally, 
the drift of $\lambda$ never stops, and the expected decrease of population size is still available when $\lambda$ exceeds the $C_0 \lambda^*$. 
	
Summarizing, the overall cost of a \emph{long} phase is 
\begin{equation*}
\sum_{t=1}^{\infty} E[T \mid I = t] \Pr[I = t] \leq D \lambda^* \sum_{t=1}^{\infty} F^{t/4}e^{-ct,}
\end{equation*}
which, having $F^{1/4}<e^{c}$, is $O(\lambda^*)$.
\end{proof}

\section{Approximation of Performance Assuming Optimal Parameter Choices}\label{sect:best}

Our experiments conducted for understanding the performance landscape of the $(1+(\lambda,\lambda))$~GA in Section~\ref{sect:eui}
could be used to derive the best values of $\lambda$ as a function of the Hamming distance to the optimum. Using these values of $\lambda$,
we may estimate what the performance of the $(1+(\lambda,\lambda))$~GA will be assuming the optimal choice of $\lambda$ throughout the
entire run. While this has been previously done on \textsc{OneMax} in~\cite{learning-from-black-box-thcs},
there was no similar analysis, either theoretical or experimental, for other use cases considered in this paper.

We ran, for $n=1000$ and all the problems (\textsc{OneMax}, $\textsc{LinInt}_2$, $\textsc{LinInt}_5$, $\textsc{LinInt}_n$, MAX-SAT),
all the algorithms tested in Section~\ref{sect:practice}, as well as the $(1+(\lambda,\lambda))$~GA that chooses $\lambda$
based on the Hamming distance to the optimum according to the experimental results presented in Section~\ref{sect:eui}.
The latter is termed the \emph{optimally extrapolated} $(1+(\lambda,\lambda))$~GA.

The results are presented in Table~\ref{table-optimal}.
One can see that, unlike the self-adjusting versions of the $(1+(\lambda,\lambda))$~GA, including the one proposed in this paper,
are outperformed by the optimally extrapolated one. In particular, the latter also outperforms the $(1+1)$~EA on all the problems,
including the hard $\textsc{LinInt}_n$. This observation, although being not very surprising, demonstrates that there is some room for
improvement even in the unchanged framework of the $(1+(\lambda,\lambda))$~GA.

\begin{table}[!t]
\caption{Comparison between the existing algorithms and the optimally extrapolated $(1+(\lambda,\lambda))$~GA
on various problems, problem size $n=1000$}\label{table-optimal}
\centering
\newcommand{\phh}{\phantom{1}}
		\begin{tabular}{c|l|l}
			Problem 			& Algorithm 					   & Fitness Evaluations \\
			\hline
			& RLS 							   & $\phh6569.05 \pm 1220.90$ \\
			& (1+1)EA 						   & $10878.39 \pm 2436.27$ \\
			& $(1+(\lambda,\lambda)), n$       & $\phh6411.01 \pm \phh414.87$ \\
			\textsc{OneMax} & $(1+(\lambda,\lambda)), \log n$  & $\phh6605.57 \pm \phh590.99$ \\
			& $(1+(\lambda,\lambda)), n*$      & $\phh9257.98 \pm 2201.33$ \\
			& $(1+(\lambda,\lambda)), \log n*$ & $\phh9746.43 \pm 2631.36$ \\
			& $(1+(\lambda,\lambda)), \text{extra}$ 	   & $\phh5651.37 \pm \phh413.71$ \\
			\hline					  
			& RLS 							   & $\phh6669.93 \pm 1126.18$ \\
			& (1+1)EA 						   & $10909.74 \pm 1982.30$ \\
			& $(1+(\lambda,\lambda)), n$       & $\phh9183.77 \pm 2266.10$ \\
			$\textsc{LinInt}_2$ & $(1+(\lambda,\lambda)), \log n$  & $\phh8066.27 \pm 1093.61$ \\
			& $(1+(\lambda,\lambda)), n*$      & $\phh9906.53 \pm 2295.75$ \\
			& $(1+(\lambda,\lambda)), \log n*$ & $\phh9923.59 \pm 1907.57$ \\
			& $(1+(\lambda,\lambda)), \text{extra}$ 	   & $\phh7531.25 \pm 1186.80$ \\ 
			\hline					  
			& RLS 							   & $\phh6864.09 \pm 1328.93$ \\
			& (1+1)EA 						   & $10688.14 \pm 2347.31$ \\
			& $(1+(\lambda,\lambda)), n$       & $12624.72 \pm 3266.02$ \\
			$\textsc{LinInt}_5$ & $(1+(\lambda,\lambda)), \log n$  & $10403.99 \pm 1696.30$ \\
			& $(1+(\lambda,\lambda)), n*$      & $11144.12 \pm 2145.53$ \\
			& $(1+(\lambda,\lambda)), \log n*$ & $11027.08 \pm 1996.54$ \\
			& $(1+(\lambda,\lambda)), \text{extra}$ 	   & $\phh9502.26 \pm 1632.21$ \\
			\hline					
			& RLS 							    & $\phh6773.77 \pm 1379.30$ \\
			& (1+1)EA 						    & $11216.71 \pm 2414.28$ \\
			& $(1+(\lambda,\lambda)), n$       & $15420.16 \pm 4281.11$ \\
			$\textsc{LinInt}_{1000}$ & $(1+(\lambda,\lambda)), \log n$  & $12756.18 \pm 2703.10$ \\
			& $(1+(\lambda,\lambda)), n*$      & $12358.61 \pm 2631.16$ \\
			& $(1+(\lambda,\lambda)), \log n*$ & $12280.56 \pm 2367.41$ \\
			& $(1+(\lambda,\lambda)), \text{extra}$ 	& $10994.36 \pm 2330.01$ \\ 
			\hline 
			& RLS 							    & $\phh7458.58 \pm 1392.33$ \\
			& (1+1)EA 						    & $11087.48 \pm 2094.26$ \\
			& $(1+(\lambda,\lambda)), n$       & $\phh9821.03 \pm 3086.19$ \\
			MAX-SAT 				 & $(1+(\lambda,\lambda)), \log n$  & $\phh8442.38 \pm 1044.57$ \\
			& $(1+(\lambda,\lambda)), n*$      & $10337.16 \pm 2018.62$ \\
			& $(1+(\lambda,\lambda)), \log n*$ & $10558.63 \pm 2140.69$ \\
			& $(1+(\lambda,\lambda)), \text{extra}$ 	& $\phh7805.58 \pm \phh946.41$ 
		\end{tabular}
\end{table}

\section{Conclusion}\label{sect:conclusion}

We proposed a modification of the one fifth rule, that is used for self-adjustment of the parameter $\lambda$ in the $(1+(\lambda,\lambda))$ genetic algorithm.
It is aimed at reducing the unwanted effects, resulting in the decreased performance on problems with either imperfect fitness-distance correlation
(linear functions with random integer weights limited by a constant), low to none fitness-distance correlation (linear functions with linear random weights),
or more tricky cases (random MAX-SAT problems with planted solutions and logarithmic clause densities).
Its main idea is to slow down the growth of $\lambda$ on long series of unsuccessful iterations, which, in terms of iterations, happens at roughly the square root of its original speed.

While definitely not being a silver bullet, and without a strict evidence of asymptotic speedups in cases pathological to the original self-adjustment scheme,
we still were able to see stable improvements over the classic $(1+(\lambda,\lambda))$~GA on all problematic functions. While on \textsc{OneMax} the proposed
strategy works by maybe 10\% worse, the runtime is still linear not only in practice, but also in theory, which we proved mostly along the lines
of the proof for the original self-adjustment scheme.

We also investigated how the $(1+(\lambda,\lambda))$~GA would have performed when, roughly speaking, the strategy of self-adjustment for $\lambda$ is optimal.
For that we derived the optimal values of $\lambda$, as a function of the Hamming distance to the optimum, from the previous experimental results,
and used them in the subsequent runs of the $(1+(\lambda,\lambda))$~GA. Although this does not form a realistic evolutionary algorithm, these results show
that better self-adjustment strategies can still, in theory, bring the runtimes below the level of the $(1+1)$~EA.

\textbf{Acknowledgment:} This research was supported by the Russian Scientific Foundation, agreement No.~17-71-20178.

\bibliographystyle{splncs04}
\bibliography{../../../../bibliography}

\end{document}